\documentclass[runningheads]{llncs}
\usepackage[T1]{fontenc}
\usepackage[utf8]{inputenc}
\usepackage{caption}
\usepackage{graphicx}
\usepackage{amsmath}
\usepackage{booktabs}
\usepackage{algorithm}
\usepackage{algorithmic}
\usepackage{url}
\usepackage{color}

\usepackage[hidelinks]{hyperref}
\usepackage{balance} 
\usepackage{latexsym}
\usepackage{enumitem}
\usepackage{subcaption}
\usepackage{amsfonts}
\usepackage{pgfplots}
\usepackage{xcolor}
\usepgfplotslibrary{statistics}
\usepackage{anyfontsize}
\usepackage{glossaries}
\glsdisablehyper
\usepackage{multirow}
\usepackage{nicefrac}

\usepackage{pgfplots}
\usepackage{orcidlink}
\usepackage{pgfplotstable}
\usetikzlibrary{patterns}
\pgfplotsset{compat=1.17}
\usetikzlibrary{calc}
\usetikzlibrary{fit}
\usetikzlibrary{arrows,arrows.meta}
\usetikzlibrary{decorations.pathreplacing}
\usepackage{adjustbox}
\usepackage{tabularx,booktabs}
\newcolumntype{Y}{>{\centering\arraybackslash}X}
\usepackage{newfloat}
\usepackage{listings}
\floatstyle{ruled}
\newfloat{listing}{tb}{lst}{}
\floatname{listing}{Listing}

\newacronym{CP}{CP}{\emph{Career Planning}}
\newacronym{cacgs}{CACGs}{\emph{Computer-Assisted Career Guidance Systems}}
\newacronym{AI}{AI}{Artificial Intelligence}
\newacronym{CPP}{DPP}{\emph{Degree Planning Problem}}
\newacronym{qp}{QP}{\emph{Quadratic Programming}}
\newacronym{mip}{MIP}{\emph{Mixed-Integer Programming}}
\newacronym{miqp}{MIQP}{\emph{Mixed-Integer Quadratic Programming}}
\newacronym{miqcp}{MIQCP}{\emph{Mixed-Integer Quadratically Constrained Programming}}
\newacronym{milp}{MILP}{\emph{Mixed-Integer Linear Programming}}
\newacronym{ilp}{ILP}{\emph{Integer Linear Programming}}
\newacronym{uctp}{UCTP}{\emph{University-Course Timetabling Problem}}
\newacronym{csp}{CSP}{\emph{Constraint Satisfaction Problem}}
\newacronym{rcpsp}{RCPSP}{\emph{Resource-Constraint Project Scheduling Problem}}
\newacronym{rpp}{RPP}{\emph{Research Project Planning}}
\newacronym{jssp}{JSSP}{\emph{Job-Shop Scheduling Problem}}
\newacronym{wap}{WAP}{\emph{Workforce Assignment Problem}}
\newacronym{ga}{GA}{\emph{Genetic Algorithm}}
\newacronym{sa}{SA}{\emph{Simulated Annealing}}
\newacronym{ts}{TS}{\emph{Tabu Search}}
\newacronym{pso}{PSO}{\emph{Particle Swarm Optimization}}
\newacronym{ls}{LS}{\emph{Least Squares}}
\newacronym{lp}{LP}{\emph{Linear Programming}}
\newacronym{cp}{CP}{\emph{Constraint Programming}}
\newacronym{socp}{SOCP}{\emph{Second-Order Cone Programming}}
\newacronym{sfia}{SFIA}{\emph{Skills Framework for the Information Age}}
\newacronym{ers}{EdRecSys}{\emph{Educational Recommender Systems}}
\newacronym{er}{EdRecSys}{\emph{Educational Recommender System}}
\newacronym{rs}{RecSys}{\emph{Recommender Systems}}
\newacronym{iis}{IIS}{\emph{Irreducible Infeasible Subsystem}}
\newacronym{mus}{MUS}{\emph{Minimal Unsatisfiable Subset}}
\newacronym{mss}{MUS}{\emph{Maximal Satisfiable Subset}}
\newacronym{mcs}{MCS}{\emph{Minimal Correction Subset}}
\newacronym{mfs}{MFS}{\emph{Maximal Feasible Subsystem}}
\newacronym{tel}{TEL}{\emph{Technology Enhanced Learning}}
\newacronym{ai}{AI}{\emph{Artificial Intelligence}}
\newacronym{xai}{XAI}{\emph{eXplainable Artificial Intelligence}}
\newacronym{sdg}{SDG}{\emph{Sustainable Development Goals}}
\newacronym{gfi}{GFI}{\emph{Gunning Fog Index}}
\newacronym{exmip}{X-MILP}{\emph{Explainable MILP}}
\newacronym{asp}{UDSP}{\emph{User-Desired Satisfiability Problem}}
\newacronym{ml}{ML}{\emph{Machine Learning}}
\newacronym{ca}{CA}{\emph{Combinatorial Auction}}
\newacronym{cap}{CAP}{\emph{Combinatorial Auction Problem}}
\newacronym{wsp}{WSP}{\emph{Weighted Set Packing}}
\newacronym{wdp}{WDP}{\emph{Winner Determination Problem}}
\newacronym{psplib}{PSPLIB}{\emph{Project Scheduling Problem Library}}
\newacronym{cats}{CATS}{\emph{Combinatorial Auction Test Suite}}
\newacronym{llm}{LLM}{\emph{Large Language Models}}
\newacronym{nl}{NL}{\emph{natural language}}

\newacronym{bfs}{BFS}{\emph{Breadth-First Search}}

\newacronym{dfs}{DFS}{\emph{Depth-First Search}}

\begin{document}
\title{Exploiting Constraint Reasoning\\to Build Graphical Explanations\\for Mixed-Integer Linear Programming}
\titlerunning{Graphical Explanations for MILP}
%
\author{Roger X. Lera-Leri\thanks{Corresponding author}  \inst{1}\orcidlink{0000-0003-4981-8260} \and
Filippo Bistaffa\inst{1}\orcidlink{0000-0003-1658-6125} \and
Athina Georgara\inst{2}\orcidlink{0000-0001-5992-5372} \and
Juan A. Rodríguez-Aguilar\inst{1}\orcidlink{0000-0002-2940-6886}}
\authorrunning{R. X. Lera-Leri et al.}
%
\institute{Artificial Intelligence Research Institute (IIIA-CSIC), Barcelona, Spain
\email{\{rlera,filippo.bistaffa,jar\}@iiia.csic.es}\and
University of Southampton, Southampton, United Kingdom
\email{a.georgara@soton.ac.uk}}
\maketitle              
\begin{abstract}
Following the recent push for \emph{trustworthy} \acrshort{ai}, there has been an increasing interest in developing \emph{contrastive explanation} techniques for optimisation, especially concerning the solution of specific decision-making processes formalised as \acrshort{milp}s. 
Along these lines, we propose \acrshort{exmip}, a domain-agnostic approach for building contrastive explanations for \acrshort{milp}s based on constraint reasoning techniques. 
First, we show how to encode the queries a user makes about the solution of an \acrshort{milp} problem as additional constraints.
Then, we determine the reasons that constitute the answer to the user's query by computing the \gls{iis} of the newly obtained set of constraints. 
Finally, we represent our explanation as a ``graph of reasons'' constructed from the \gls{iis}, which helps the user understand the structure among the reasons that answer their query. 
We test our method on instances of well-known optimisation problems to evaluate the \emph{empirical hardness} of computing explanations.

\end{abstract}

\section{Introduction}

In recent years, \gls{ai} has achieved notable momentum, enabling impressive results in many application domains. 
Alongside such improvements, there has been an increasing interest in developing trustworthy, human-centred AI that aligns with ethical values. A key requirement is \emph{explainability}, namely, providing explanations to humans who use or are affected by AI systems \cite{URL-EU-TrustworthyAI21}.
Consequently, there has been a resurgence in the area of \gls{xai}, 
which is concerned with designing \gls{ai} systems whose decisions can be explained and understood by humans \cite{rai2020explainable}.

\gls{xai} has been widely active in \gls{ml} to provide insights into the functionalities of models that operate as ``black boxes''. 
Following this push on \gls{xai}, there have been recent works on explainability for optimisation in domain-specific applications (e.g., scheduling \cite{pozanco2022explaining} or team formation \cite{georgara2022building}). 

Here we consider a domain-independent approach where the optimal solution of an optimisation problem formalised as a \gls{milp} is presented to a user, who then asks for an explanation of such a solution.
Based on the user query, the literature typically computes an alternative 
solution compliant with such a query, the so-called user-desired solution \cite{georgara2022building,zehtabi2024contrastive}. 
Then, it provides a contrastive explanation comparing the optimal solution with the user-desired solution. 
In contrast, our goal is to provide a \emph{contrastive explanation} to the user, presenting the \emph{reasons} for which their query leads to worse decisions in terms of optimality.

To achieve this goal, we propose \gls{exmip}, a general method that aims to provide contrastive explanations for \gls{milp} problems (see Figure \ref{fig:overview}). 
\gls{exmip} is a process that follows three main steps.
First, it casts a user-desired optimisation scenario compliant with a user's query as a \gls{csp}. 
Second, it finds the constraints in the \gls{csp} that make it infeasible. The infeasibility of such \gls{csp} proves that the user-desired optimisation scenario cannot lead to a better solution than the challenged optimal solution of the \gls{milp}. With this goal, our approach exploits constraint reasoning techniques to calculate the minimal subset of constraints that make the \gls{csp} infeasible, i.e., the \gls{iis} \cite{chinneck2007feasibility}. Third, \gls{exmip} builds a structured, human-readable explanation from the unstructured collection of mathematical constraints in the \gls{iis}. 
Such explanation, the so-called \emph{graph of reasons}, graphically captures the relationships between constraints and assigns them a \acrlong{nl} interpretation, which we call ``\emph{reason}''.

To the best of our knowledge, the novelty of our work stems from the following contributions: 
\begin{itemize}
    \item We cast the problem of generating a contrastive explanation for the result of an optimisation as a \gls{csp}, which encodes a user-desired scenario compliant with a query posed by a user. 
    \item We exploit constraint reasoning techniques to find why a user-desired optimisation scenario cannot lead to a better solution than the optimal solution. We do so by computing the minimal set of constraints, i.e., an \gls{iis}, in the \gls{csp} encoding the user-desired scenario. 
    \item We build a graphical explanation based on the information in the \gls{iis}, whose structure allows to capture the relationships between constraints that an unstructured set of constraints cannot. 
    \item We evaluate \gls{exmip} on different \gls{milp} problems to showcase its generality. We test our approach on a scheduling problem, the so-called \gls{rcpsp}, and a particular instance of the \gls{wsp} problem. 
    Our results show that despite the hardness of solving these optimisation problems, \gls{exmip} computes explanations in less than $60$ seconds in most cases. Moreover,  \gls{exmip} does not recompute the solution to an optimisation problem, in contrast with previous works \cite{zehtabi2024contrastive,georgara2022building}. 
\end{itemize}

\tikzstyle{box}=[
    draw,
    shape=rectangle,
    inner sep=0,
    minimum width=1.75cm,
    minimum height=1cm,
    align=center,
]
\tikzstyle{gnode}=[
    draw,
    shape=circle,
    inner sep=0,
    minimum size=1.0cm,
    align=center,
]

\begin{figure}[t]
\centering
\begin{adjustbox}{max width=0.90\columnwidth}
\begin{tikzpicture}[
    every node/.style={outer sep=0},
    yscale=0.8,
]


\node[label={\small User}] at (0,0) (user) {\includegraphics[width=8mm]{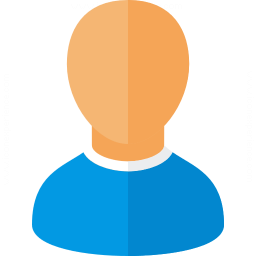}};

\node[box,minimum width=15mm,minimum height=9mm,rounded corners,very thick] at ($(user)+(0,2.1)$) (mip) {\small MILP\\\small Problem};


\node[box,minimum width=30mm,minimum height=9mm,rounded corners,label=below:{\small User Query ($q$)},very thick,font=\itshape] at ($(user) - (3,0)$) (query) {\small Why not this\\other solution $S'$?};

\node[label={\small Solver}] at ($(mip) + (2.6,0)$) (solver) {\includegraphics[width=10mm]{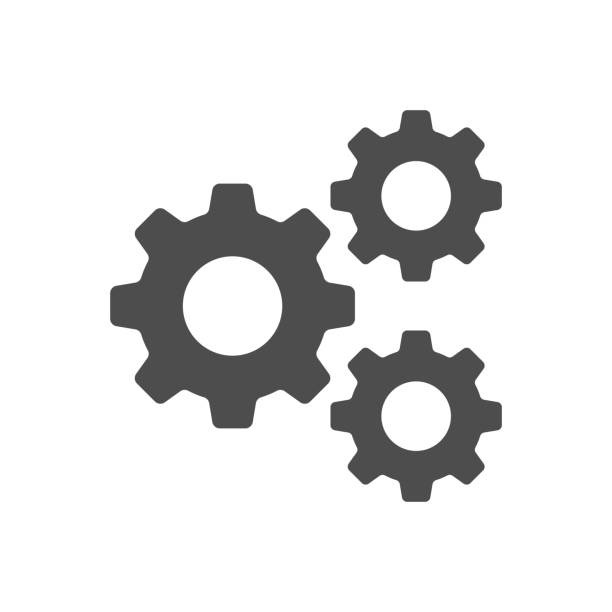}};

\node[box,minimum width=18mm,minimum height=9mm,rounded corners,very thick] at (solver|-user) (solution) {\small Solution $S$};


\node[box,minimum width=20mm,minimum height=10mm,rounded corners,very thick] at ($(user) + (-7,2.5)$) (QT) {\small Query\\\small Translation};

\node[box,minimum width=20mm,minimum height=12mm,rounded corners,very thick] at ($(QT) - (0,1.9)$) (ASP) {\small User-Desired\\\small Satisfiability\\\small Problem};

\node[box,minimum width=20mm,minimum height=12mm,rounded corners,very thick] at ($(ASP) - (0,2)$) (IIS) {\small Irreducible\\\small Infeasible\\\small Subsystem};

\node[box,minimum width=20mm,minimum height=10mm,rounded corners,very thick] at ($(IIS) - (0,1.9)$) (reasons) {\small Graph of\\\small Reasons};

\node[draw=black,very thick,fit=(QT) (reasons),inner sep=2.0mm,label=above:{\small Explainable MILP}] (exmip) {};

\node[box,rounded corners,minimum width=6mm,minimum height=6mm,very thick] at ($(user) - (2.55,2.7)$) (q0) {\small $q$};
\node[box,rounded corners,minimum width=16mm,minimum height=6mm,very thick] at ($(q0) + (2,0.5)$) (r1) {\small \emph{Reason 1}};
\node[box,rounded corners,minimum width=16mm,minimum height=6mm,very thick] at ($(q0) + (2,-0.5)$) (r2) {\small \emph{Reason 2}};
\node[box,rounded corners,minimum width=16mm,minimum height=6mm,very thick] at ($(r1) + (2.6,0)$) (r3) {\small \emph{Reason 3}};
\node[box,rounded corners,minimum width=16mm,minimum height=6mm,very thick] at (r3|-r2) (r4) {\small \emph{Reason 4}};

\path[-,draw,black] (q0.east) to (r1.west);
\path[-,draw,black] (q0.east) to (r2.west);
\path[-,draw,black] (r2.east) to (r3.west);
\path[-,draw,black] (r2.east) to (r4.west);
\path[-,draw,black] (r1.south) to (r2.north);

\node[draw=black,very thick,fit=(q0)(r1)(r4),inner sep=2.5mm,label=below:{\small Explanation}] (explanation) {};


\path[-{Stealth},draw,black] (user.west) to (query.east);
\path[-{Stealth},draw,black] (mip.west) to (ASP.east);
\path[-{Stealth},draw,black,out=90,in=0] (query.north) to (QT.east);
\path[-{Stealth},draw,black] (QT.south) to (ASP.north);
\path[-{Stealth},draw,black] (ASP.south) to (IIS.north);
\path[-{Stealth},draw,black] (IIS.south) to (reasons.north);
\path[-{Stealth},draw,black] (reasons.east) -- (explanation.west) node[midway, above] {};
\path[-{Stealth},draw,black] (explanation.north) to (user.south);
\path[-{Stealth},draw,black] (solution.west) to (user.east);
\path[-{Stealth},draw,black] (mip.east) to (solver.west);
\path[-{Stealth},draw,black] (solver.south) to (solution.north);

\end{tikzpicture}
\end{adjustbox}
\caption{Overview of our approach. See Section \ref{sec:expl} for the details.}
\label{fig:overview}
\end{figure}
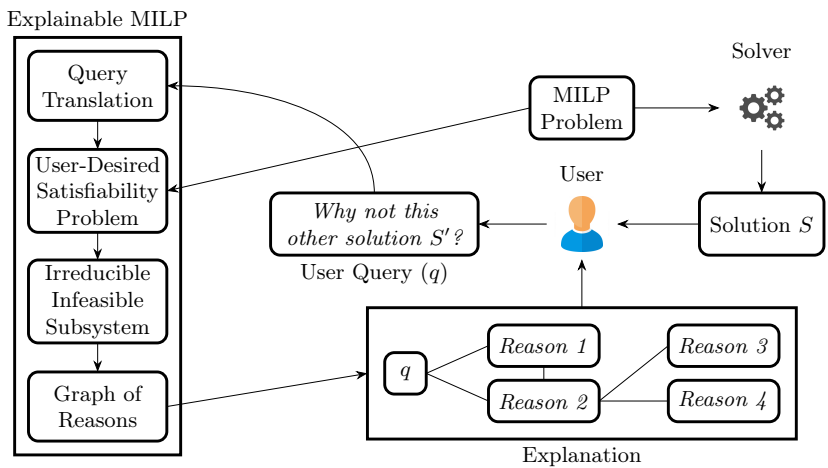

\section{Background}
\label{sec:back}

\subsection{\acrlong{milp}}
\label{subsec:mip}

An \gls{milp} is an optimisation problem in the form
\begin{equation}\label{prob:mip}
    \mbox{minimise} \quad \textbf{c}^{T}\cdot\textbf{x} + \textbf{h}^{T}\cdot\textbf{y}, \quad
    \mbox{s.t.} \quad A\cdot \textbf{x} + G\cdot \textbf{y}\leq \textbf{b} ,
\end{equation}
where $\textbf{x}\in \mathbb{Z}^{n}$ is a vector of integer variables, $\textbf{y} \in \mathbb{R}^{p}$ is a vector of variables, $\textbf{c}\in \mathbb{R}^{n},\textbf{h}\in \mathbb{R}^{p}, \text{ and } \textbf{b}\in\mathbb{R}^{m}$ are column vectors, and $A\in \mathbb{R}^{m\times n},\text{ and } G\in \mathbb{R}^{m\times p}$ are matrices. 
In general, \gls{milp} problems are NP-Hard \cite{wolsey2020integer} and are expensive to solve. 
On the other hand, available solvers such as CPLEX optimally solve \glspl{milp}.

\subsection{Benchmark Optimisation Problems}\label{subsec:mip_prob}

We describe two well-known optimisation problems that can be formalised as \glspl{milp}, which we employ as benchmarks in Section \ref{sec:exp} to evaluate our approach. We choose them because they are examples of a scheduling problem (the \gls{rcpsp}), and a \gls{wsp} problem (the \gls{wdp} for \gls{ca}). 
Our choice of employing a \gls{wsp} formalisation is also motivated by the fact that such a problem has been used in real-world application domains such as \emph{shared mobility} \cite{bistaffa2019computational}, \emph{cooperative learning} \cite{andrejczuk2019synergistic}, and, according to \cite{FENOY2024104064}, any scenario involving the formation of collectives of agents. 
The details of their formalisation are in the supplementary material.\footnote{\url{https://filippobistaffa.github.io/files/EXTRAAMAS_2025_Suppl_Material.pdf}}

\paragraph{\gls{rcpsp}}
\gls{rcpsp} is a problem consisting of activities that must be scheduled to meet precedence and resource constraints while minimising a project's makespan \cite{hartmann2022}. 
Each activity requires a processing time and certain resources to be completed.  
Moreover, each activity has preceding activities, and hence it can only begin once they finish. 
The goal of an \gls{rcpsp} is to find a feasible schedule that minimises the total project duration.
Therefore, the decision variables of the \gls{rcpsp} are $x_{j,t} \in \{0,1\}$, which indicate whether activity $j$ is completed at time $t$. Then, an \gls{rcpsp}s' constraints are: (i) every activity must be completed, (ii) every activity must start after their predecessors have been completed, and (iii) resource availability has to be respected every time unit.

\paragraph{\gls{wdp}}
Given a combinatorial auction that receives a set of bids over a set of goods, the \gls{wdp} is that of selecting the subset of bids that maximise revenue \cite{de2003combinatorial}. 
We follow the standard approach to solve a \gls{wdp} by casting it as a \gls{wsp} problem modelled as an \gls{milp}, where decision variable $x_b \in \{0,1\}$ indicates whether bid $b$ is selected.
The \gls{wdp}'s constraints force that every good is allocated up to one winning bid.

\subsection{\acrlong{iis}}\label{subsec:iis}

The \acrfull{iis} is a minimal set of inconsistent constraints. 
However, to properly define the notion of \gls{iis}, we must first define the notions of constraint and system of constraints. We borrow the definition of constraint from \cite{meseguer2006soft}:
\begin{definition}
    Given a sequence of distinct variables $V=\langle x_1,\dots,$ $x_k\rangle$ and their associated finite domains $\mathcal{D}_1,\dots,\mathcal{D}_k$, a constraint $c$ is a subset of $\mathcal{D}_1\times\dots\times\mathcal{D}_k$. Thus, a constraint $c$ specifies the assignments allowed by the variables of $V$.
\end{definition}
In other words, a constraint is a condition that variables must satisfy. 
Then, a constraint system is a set of constraints $C$. 
Hence, we define an infeasible system as follows.
\begin{definition}
    An Infeasible System of constraints is a set of constraints $C$ such that no assignment of $V$ exists that satisfies all the constraints in $C$.
\end{definition}
Having defined an infeasible system of constraints, we define an \gls{iis} borrowing the definition from \cite{chinneck2007feasibility}:

\begin{definition}\label{def:iis}
    An \gls{iis} is a subset of constraints within an infeasible system that is in itself infeasible, but any proper subset of the \gls{iis} is feasible. 
\end{definition}

Finding an \gls{iis} is crucial in \gls{csp} 
because it identifies the constraints preventing us from obtaining a feasible solution. Thus, many algorithms have been developed to find the \gls{iis} for \gls{lp} \cite{van1981irreducibly,gleeson1990identifying,parker1996finding,chinneck1997finding}, \gls{milp} \cite{guieu1999analyzing} or Non-Linear Integer Programming \cite{chinneck1995analyzing}. Finding the \gls{iis} involves iteratively analysing its infeasible system and identifying the minimal subset of constraints responsible for infeasibility. Some examples of algorithms that find \glspl{iis} are: the Deletion Filter \cite{chinneck1991locating}, which iteratively deletes constraints and checks the feasibility of the system; the Additive Method \cite{tamiz1996detecting}, which iteratively adds constraints to an initially empty set and tests whether the set is feasible; or the combination of both methods \cite{guieu1999analyzing}.  
Although computing \glspl{iis} can be costly, state-of-the-art solvers efficiently compute them.

\section{Related Work}\label{sec:related}

Here, we discuss recent literature in the \acrlong{xai} field and explainability for optimisation. 
Notably, most \gls{xai} works focus on \gls{ml} models.
For instance, \cite{arrieta2020explainable,adadi2018peeking} provide an overview of the \gls{xai} current proposals in \gls{ml}. 
These works categorise \gls{xai} techniques based on their explanations' scope, the methodologies behind the algorithms, and their explanation level or usage. Besides, \cite{samek2019towards} point out future \gls{xai} challenges, such as formalising a theory of \gls{xai} or providing high-level explanations.

Furthermore, it is important to characterise what constitutes a \emph{good} explanation. \cite{miller2019explanation} reviews the social sciences literature that studies how humans define, generate, and evaluate explanations. \cite{miller2019explanation} extracts important findings and provides insights on how to apply such findings to \gls{xai}, concluding that giving explanations is a non-trivial task since such explanations have to be: (i) \emph{contrastive} (in response to particular counterfactual cases); (ii) \emph{social} (to transfer knowledge as a part of a conversation); (iii) \emph{selected} (humans select one or two reasons as an explanation); and (iv) \emph{causal} (explanations with causal relationships are more important than referring to probabilities or statistical relations).

Even though many AI systems use optimisation-based approaches, there is limited research on providing explanations for such systems. 
\cite{nardi2022graph} find justifications for the outcome of collective decision algorithms. 
\cite{bogaerts2021framework} give step-wise explanations for \gls{csp}.
Moreover, \cite{korikov2021counterfactual,pozanco2022explaining,georgara2022building} explain the outcome of different classical combinatorial problems. However, the explanations generated by such approaches lack some key properties enumerated by \cite{miller2019explanation}. Although \cite{korikov2021counterfactual,georgara2022building} provide contrastive explanations, they refer to numerical reasons, instead of referring to causal relationships. On the other hand, \cite{pozanco2022explaining} refer to reasons, but their explanations are not contrastive. Moreover, these approaches are not general; each one addresses a particular application domain. 
Against this background, here we propose \gls{exmip}, a domain-independent explanatory approach that generates contrastive, social, and selected explanations that refer to reasons indicating why the optimal solution is better than the one posed by the user.  

Although we employ \glspl{iis} in \gls{exmip}, \glspl{mus} are an analogous concept that has been studied in the Boolean satisfiability (SAT) literature \cite{silva2010minimal}. 
Several works discuss how to extract multiple \glspl{mus} \cite{liffiton2005finding}, 
since an infeasible system of constraints may possess an exponential number of \glspl{mus} \cite{chakravarti1994some}. 
Along these lines, \cite{junker2004quickxplain,gamba2023efficiently} propose methods to select the \gls{mus} whose constraints are the most relevant according to user preferences. 
Furthermore, there has been an effort in computing the \gls{mus} with minimal cardinality, i.e., the \gls{mus} with a minimum number of constraints \cite{liffiton2009branch,ignatiev2015smallest}. 
Crucially, the above-mentioned works on \gls{mus} focus on providing explanations for the \emph{infeasibility} of a given \emph{SAT Problem} by extracting a \gls{mus}.
In this paper, we pursue a \emph{different} 
goal. We aim to provide explanations for \glspl{milp} by answering a user's query on the optimal solution.
We do so by computing the \gls{iis} of a specially-defined \gls{csp} (see Section \ref{subsec:asp}) and identifying the \emph{reasons} that constitute such an explanation, which are then represented as a human-readable graph (see Section \ref{subsec:ordering}).

\begin{table}[b]
    \caption{Start and completion times, predecessors, and resources for some activities in the optimal solution of Example \ref{example:knapsack}. }
    \label{tab:solution_rcpsp}
    \setlength{\tabcolsep}{9pt}
    \centering
    \begin{tabularx}{0.95\columnwidth}{ccccc}
        
        \toprule
        Activities & Start time & Completion time & Predecessors & Resources \\
        \midrule
        $16$ & $14$ & $23$ & $\{10\}$ & $(0,0,0,5)$ \\
        $17$ & $24$ & $29$ & $\{13,14\}$ & $(0,0,0,8)$ \\
        $22$ & $30$ & $36$ & $\{17,16\}$ & $(2,0,0,0)$ \\
        $23$ & $37$ & $38$ & $\{22\}$ & $(3,0,0,0)$ \\
        $24$ & $39$ & $41$ & $\{23\}$ & $(0,9,0,0)$ \\
        \bottomrule
    \end{tabularx}
\end{table}

\begin{table*}[t]
    \setlength{\tabcolsep}{3pt}
    \caption{Set of possible user questions, query constraints and constraint encodings for the \gls{rcpsp} formalisation.}
    \label{tab:queries_rcpsp}
    \begin{tabularx}{\textwidth}{p{0.04\textwidth}p{0.33\textwidth}p{0.30\textwidth}Y}
        \toprule
        & \textbf{Question} & \textbf{Query constraints} & \textbf{Constraint encoding ($C_Q$)} \\
        \midrule
        Q1.& Why is activity $j$ completed at time $t$? & \textbf{veto} $j$ complete at time $t$ & ${x}_{j,t} = 0$ \\
        Q2.&  Why is activity $j$ \textbf{not} completed at time $t$? & \textbf{enforce} $j$ complete at time $t$ & ${x}_{j,t} = 1$ \\
        Q3.& Why is activity $j$ \textbf{not} completed before time $t$? & \textbf{enforce} $j$ complete before time $t$ & $\sum_{t'=EF_j}^{t-1}{x}_{j,t'} = 1$ \\
        Q4.& Why is activity $j$ \textbf{not} completed after time $t$? & \textbf{enforce} $j$ to complete after $t$ & $\sum_{t'=t+1}^{LF_j}{x}_{j,t'} = 1$ \\
        Q5.& Why is the group of activities $A$ completed at $t$? & \textbf{veto} to complete $j\in A$ at time $t$ & $\sum_{j\in A}{x}_{j,t} < |A|$ \\
        Q6.& Why is the group of activities $A$ \textbf{not} completed at $t$? & \textbf{enforce} to complete all of them at time $t$ & $\sum_{j\in A}{x}_{j,t} = |A|$ \\
        Q7.& Why is activity $j$ completed at $t$ instead of $t'$? & \textbf{enforce} complete $j$ at $t'$ and \textbf{veto} $j$ at $t$& ${x}_{j,t'} = 1$, ${x}_{j,t} = 0$ \\
        Q8.& Why is activity $j$ completed at $t$ instead of $j'$? & \textbf{enforce} $j'$ and \textbf{veto} $j$ complete at time $t$ & ${x}_{j',t} = 1$, ${x}_{j,t} = 0$ \\
        \bottomrule
    \end{tabularx}
\end{table*}
\normalsize

\section{Building Explanations for \glspl{milp}}\label{sec:expl}

We aim to provide 
contrastive explanations to users for \gls{milp} problems.\footnote{In this paper we only focus on providing explanations for optimisation problems formalised as \glspl{milp}. Explaining problems formalised as \glspl{mip}, (Max)SATs, etc., is outside the scope of this paper.} 
In such problems, constraints encode restrictions that the solution must satisfy, while the objective function encodes some preferences or quantities to be optimised \cite{boyd2004convex}. Thus, a MILP solver output is affected by the constraints and the objective function. 
Hence, we propose to exploit the information provided by the constraints and the objective function in an \gls{milp} to build contrastive explanations. 
The following running example illustrates the concepts of our approach.

\begin{example}\label{example:knapsack}
    Consider an \gls{rcpsp} instance with 30 activities and 4 different resources. 
    As mentioned in Section \ref{subsec:mip_prob}, solving an \gls{rcpsp} amounts to minimising a project duration objective function $f$, 
    while satisfying a constraint set $C$. 
    We partially show the optimal solution $\textbf{x}^{*}$ in Table \ref{tab:solution_rcpsp}. The optimal (minimum) project duration would be 43 ($f^{*} = 43$).
\end{example}

Some \gls{xai} approaches in the literature build a contrastive explanation by comparing an optimal solution with a user-desired solution \cite{georgara2022building,zehtabi2024contrastive}. The latter results from translating a user ``\emph{query}'', i.e., a question for which a user expects an explanation, into \emph{addditional constraints} that a solution must satisfy.

\subsection{Query Translation}\label{sec:query}

Here, we characterise the types of queries that a user can pose and how to encode them into query constraints. 
More precisely, a user asks ``\emph{why}'' questions about value assignments that did occur in an optimal solution $\textbf{x}^{*}$ of an \gls{milp}. 
For example, regarding a scheduling problem, a user might ask ``why is job j scheduled at a given time t?''. 
Hence, a contrastive explanation should tell the user why a solution complying with the query constraints is not as good as the optimal solution.
For that, we assume that the query a user poses refers to a \gls{milp} solver output \cite{zevcevic2021causal}. 
Therefore, a user query is related to the decision variables of the \gls{milp} at hand.
In combinatorial optimisation problems, decision variables are related to objects, activities, agents, etc. 
For instance, in the \gls{rcpsp}, the decision variables indicate whether an activity is completed at a given time. Hence, there is a decision variable for every pair of activities and time units. 

Here we follow the query constraint classification proposed in \cite{georgara2022building}:
\begin{itemize}
\item \textbf{enforce} constraints to express that some value assignments that do not hold in an optimal solution are required to hold in the user-desired solution; and 
\item \textbf{veto} constraints to express that some value assignments that do hold in an optimal solution must not hold in the user-desired solution.
\end{itemize}

Following such classification, Table \ref{tab:queries_rcpsp} lists possible queries for the \gls{rcpsp}. The table also shows the translations of queries into (linear) query constraints. 
Note that such a methodology can be applied to other \gls{milp} problems \cite{georgara2022building}. We show the list of possible queries that we identify for the \gls{wdp} in the supplementary material. Here, we discuss some examples for the \gls{rcpsp}.

In the case of the \gls{rcpsp}, for example, the query ``\emph{Why is activity $j$ scheduled to be completed at time $t$?}'' translates into \textbf{vetoing} the completion of activity $j$ at time $t$ (i.e., $x_{j,t}=0$) to enforce an alternative solution to the optimal solution. Notice that this query refers to a \textbf{single assignment} occurring in the optimal solution. 
Additionally, to comply with the query ``\emph{Why is the group of activities $A$ scheduled to be completed at time $t$?}'' we must \textbf{veto} the completion of all activities in $A$ at $t$ (i.e., $\sum_{j \in A} x_{j,t}=0$). In this case, the query refers to the schedule of multiple activities, i.e., \textbf{multiple assignments}, in the optimal solution.
Furthermore, when addressing the query ``\emph{Why is activity $j$ \textbf{not} scheduled to be completed at time $t$?}'' we must \textbf{enforce} the hypothetical scheduling of activity $j$, which the current scheduling disregards.
Moreover, to answer the query ``\emph{Why is activity $j$ scheduled to be completed at time $t$ instead of activity $j'$?}'', we must: (1) enforce activity $j'$ ; and (2) veto $j$ at time $t$. Thus, we consider combinations of \textbf{enforce} and \textbf{veto} constraints to answer more specific queries about value assignments that did and did not occur.

\begin{example}\label{example:KP2}
    Following Example \ref{example:knapsack}, a user may ask---``\emph{Why is activity 24 not completed before time 41?}''---about the solution in Table \ref{tab:solution_rcpsp}. Building an explanation to this query requires enforcing activity 24 to be completed before time 41, which is encoded by constraint   \begin{equation}\label{eq:rcpsp:query:example}
        \sum_{t=1}^{40} x_{24,t} = 1,
    \end{equation} 
    (see Table \ref{tab:queries_rcpsp}). 
    Enforcing this constraint would allow to obtain a user-desired solution. 
\end{example}

\subsection{The User-Desired Satisfiability Problem}\label{subsec:asp}

Based on a user's query, the literature typically considers user-desired optimisation scenarios that accommodate such a query, 
aiming at quantifying how worse a user-desired solution would be in terms of optimality. 
In this context, the literature typically defines the following extended problem \cite{zehtabi2024contrastive,georgara2022building}.

\begin{definition}
    Given a \gls{milp} problem $M=\langle f, C \rangle$ of the form of \eqref{prob:mip} (we refer to $M$ as \emph{main problem}, where $f$ is the objective function and $C$ are the constraints), and a set of query constraints $C_Q$ associated with a query, we obtain the \emph{extended problem} $M'=\langle f, C \cup C_Q \rangle$ by adding $C_Q$.
\end{definition}

\noindent
Assuming $M$ is a feasible problem, solving an extended problem $M'$ would result in one of the following three cases.

\begin{enumerate}\label{enum:cases}
    \item \label{enum:case_inf} \textbf{Infeasibility.} There is no guarantee that the extended problem is feasible. Hence, adding the query constraints makes the new problem $M'$ infeasible. 
    \item \label{enum:case_opt} \textbf{Optimality.} We obtain a solution whose value is the same as the optimal solution for $M$. This is possible when the main problem has multiple optimal solutions. 
    \item \label{enum:case_sub} \textbf{Suboptimality.} This occurs when the optimal solution to $M'$ is worse than the optimal solution to $M$. 
\end{enumerate}

However, in this paper, we neither solve an extended \gls{milp} problem that accommodates a user's query (as in \cite{zehtabi2024contrastive}) nor quantify how worse the user-desired solution is 
(as in \cite{georgara2022building}). 
Since numerical comparisons are not desirable in explanations \cite{miller2019explanation}, we propose instead to find the reasons why the user-desired scenario cannot lead to a better solution than the optimal solution. 
To do so, we cast the problem of obtaining the user-desired solution as a \gls{csp}. Such a \gls{csp} contains the problem constraints, the query constraints and an additional constraint enforcing a solution at least as good as the optimal solution to the ``\emph{main}'' problem. 
This formulation will serve to extract the constraints that restrict us from obtaining a user-desired solution with optimal value. 

\begin{definition}\label{def:asp}
    Let $M$ be a \gls{milp}, $C_Q$ a set of query constraints, and $f^{*}=f(\textbf{x}^{*})$ the value of the optimal solution $\textbf{x}^{*}$ for $M$.The \gls{asp} for $M$ is a \gls{csp} composed of the following constraints: 
\begin{equation}\label{eq:asp}
    (i)\; f(\textbf{x}) \leq f^{*}, \quad (ii)\; C, \quad (iii)\; C_Q,
\end{equation}
where $\textbf{x}$ are variables and (i) is a minimality constraint. 
\end{definition}
\noindent
Without loss of generality, we assume that $M$ is a minimisation problem, and hence the value of a better solution must be lower than $f^{*}$.

\begin{example}\label{example:ASP}
We define a \gls{asp} for example \ref{example:KP2} as the following \gls{csp} with variables $\textbf{x}$:    
\begin{equation}\label{eq:asp:example}
    (i)\; f(\textbf{x}) \leq f^{*} = 43, \quad (ii)\; C, \quad (iii)\; \sum_{t=1}^{40} x_{24,t} = 1,
\end{equation}
    where $(i)$ is the minimality constraint, $(ii)$ are the main problem constraints, and $(iii)$ is the query constraint.
\end{example}

\subsection{Computing the Infeasible Constraints}\label{subsec:comp_iis}

Having defined the \gls{asp}, we extract the minimal set of constraints that prevent us from finding a feasible solution to the \gls{asp}, i.e., we extract the \gls{iis}.
Notice that there might be multiple \glspl{iis} corresponding to the same set of constraints since the number of \glspl{iis} increases exponentially with the size of a constraint set \cite{chakravarti1994some}. 
More precisely, different \glspl{iis} (despite being all ``irreducible'') could contain different numbers of constraints (i.e., some \glspl{iis} might be smaller than others). 
Hence, this would subsequently lead us to explanations of varying sizes, since we build our explanations based on the constraints of an \gls{iis}.  
Following \cite{rosenfeld2021better},\footnote{The number of ``rules'' (i.e., constraints in our case) in an explanation is a commonly used metric to evaluate its simplicity.} 
we should aim to compute the smallest possible \gls{iis} because the more concise an explanation, the easier to understand.

On the one hand, some algorithms in the literature can compute the ``smallest'' \gls{iis} \cite{liffiton2009branch,ignatiev2015smallest}. However, computing the ``smallest'' \gls{iis} is very costly.
On the other hand, off-the-shelf solvers such as CPLEX use specialised heuristics to efficiently compute a single \gls{iis} that is not guaranteed to be the smallest one. 
In Section \ref{subsec:results:smalliis}, we experimentally compare the solutions computed by an algorithm for computing the smallest \gls{iis},  adapted from \cite{ignatiev2015smallest}, and CPLEX. 
We observe that, in practice, CPLEX computes the smallest \gls{iis} for most instances with a significantly lower computational cost.

\tikzstyle{box}=[
draw,
shape=rectangle,
inner sep=0,
minimum width=1.75cm,
minimum height=1cm,
align=center,
]
\tikzstyle{gnode}=[
draw,
shape=circle,
inner sep=0,
minimum size=1.0mm,
align=center,
]

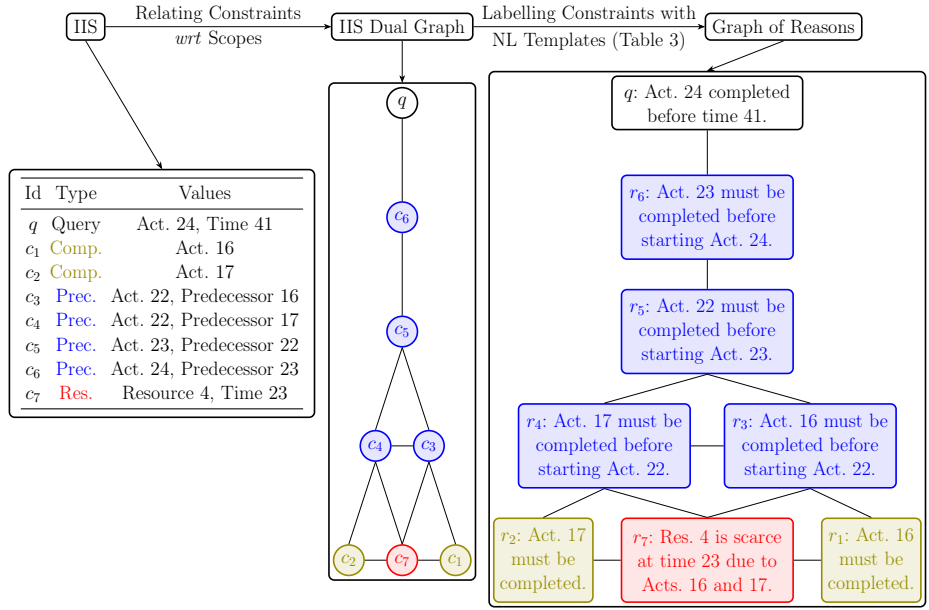
\begin{figure}[t]
\centering

\begin{adjustbox}{max width=\textwidth}
\begin{tikzpicture}

\tikzset{every node/.style={font=\Large}}
\node[align=center] at (0,0) (table) {
\centering
\setlength{\tabcolsep}{3pt}
\begin{tabular}{ccc}
\toprule
Id & Type & Values \\
\midrule
$q$ & Query & Act. $24$, Time $41$ \\
$c_1$ & \textcolor{olive}{Comp.} & Act. $16$  \\
$c_2$ & \textcolor{olive}{Comp.} & Act. $17$  \\
$c_3$ & \textcolor{blue}{Prec.} & Act. $22$, Predecessor $16$ \\
$c_4$ & \textcolor{blue}{Prec.} & Act. $22$, Predecessor $17$ \\
$c_5$ & \textcolor{blue}{Prec.} & Act. $23$, Predecessor $22$ \\
$c_6$ & \textcolor{blue}{Prec.} & Act. $24$, Predecessor $23$ \\
$c_7$ & \textcolor{red}{Res.} & Resource $4$, Time $23$ \\
\bottomrule
\end{tabular}
};

\node[box,rounded corners,minimum width=80mm,minimum height=65mm,very thick] at ($(table)+(0,0)$) (aroundtable) {};

\node[gnode,minimum size=8mm,very thick] at ($(table)+(6.3,5.0)$) (q) {$q$};

\node[gnode,minimum size=8mm,color=blue,fill=blue!10,very thick] at ($(q)+(0,-3)$) (c6) {$c_6$};

\node[gnode,minimum size=8mm,color=blue,fill=blue!10,very thick] at ($(c6)+(0,-3)$) (c5) {$c_5$};

\node[gnode,minimum size=8mm,color=blue,fill=blue!10,very thick] at ($(c5)+(0.7,-3)$) (c3) {$c_3$};

\node[gnode,minimum size=8mm,color=blue,fill=blue!10,very thick] at ($(c5)+(-0.7,-3)$) (c4) {$c_4$};

\node[gnode,minimum size=8mm,color=red,fill=red!10,very thick] at ($(c5)+(0,-6)$) (c7) {$c_7$};

\node[gnode,minimum size=8mm,color=olive,fill=olive!10,very thick] at ($(c5)+(-1.4,-6)$) (c2) {$c_2$};

\node[gnode,minimum size=8mm,color=olive,fill=olive!10,very thick] at ($(c5)+(1.4,-6)$) (c1) {$c_1$};

\node[box,rounded corners,very thick,align=center,fit=(q)(c1)(c2)(c7),inner sep=1.0mm] (iisgraph) {};


\path[-,draw,black] (q.south) to (c6.north);
\path[-,draw,black] (c6.south) to (c5.north);
\path[-,draw,black] (c5.south) to (c3.north);
\path[-,draw,black] (c5.south) to (c4.north);
\path[-,draw,black] (c4.east) to (c3.west);
\path[-,draw,black] (c3.south) to (c1.north);
\path[-,draw,black] (c3.south) to (c7.north);
\path[-,draw,black] (c4.south) to (c2.north);
\path[-,draw,black] (c4.south) to (c7.north);
\path[-,draw,black] (c1.west) to (c7.east);
\path[-,draw,black] (c7.west) to (c2.east);

\node[box,rounded corners,very thick, minimum width=50mm, minimum height=14mm] at ($(q)+(8.0,0)$) (qr) {$q$: Act. 24 completed\\ before time 41.};

\node[box,rounded corners, color=blue,fill=blue!10,very thick, minimum width=45mm,minimum height=22mm] at ($(qr)+(0,-3)$) (r6) {$r_6$: Act. 23 must be \\ completed before \\ starting Act. 24.};

\node[box,rounded corners,color=blue,fill=blue!10,very thick, minimum width=45mm,minimum height=22mm] at ($(r6)+(0,-3)$) (r5) {$r_5$: Act. 22 must be \\ completed before \\ starting Act. 23.};

\node[box,rounded corners,color=blue,fill=blue!10,very thick, minimum width=45mm,minimum height=22mm] at ($(r5)+(2.7,-3)$) (r3) {$r_3$: Act. 16 must be \\ completed before \\ starting Act. 22.};

\node[box,rounded corners,color=blue,fill=blue!10,very thick, minimum width=45mm,minimum height=22mm] at ($(r5)+(-2.7,-3)$) (r4) {$r_4$: Act. 17 must be \\ completed before \\ starting Act. 22.};

\node[box,rounded corners,color=red,fill=red!10,very thick, minimum width=45mm,minimum height=22mm
] at ($(r5)+(0,-6.0)$) (r7) {$r_7$: Res. 4 is scarce \\ at time 23 due to \\ Acts. 16 and 17.};

\node[box,rounded corners,color=olive,fill=olive!10,very thick, minimum width=26mm,minimum height=22mm] at ($(r5)+(-4.3,-6.0)$) (r2) {$r_2$: Act. 17 \\ must be \\ completed.};

\node[box,rounded corners,color=olive,fill=olive!10,very thick, minimum width=26mm,minimum height=22mm] at ($(r5)+(4.3,-6.0)$) (r1) {$r_1$: Act. 16 \\ must be \\ completed.};

\node[box,rounded corners,very thick, align=center,fit=(qr)(r1)(r2)(r7),inner sep=1.0mm] (reasongraph) {};


\path[-,draw,black] (qr.south) to (r6.north);
\path[-,draw,black] (r6.south) to (r5.north);
\path[-,draw,black] (r5.south) to (r3.north);
\path[-,draw,black] (r5.south) to (r4.north);
\path[-,draw,black] (r4.east) to (r3.west);
\path[-,draw,black] (r3.south) to (r1.north);
\path[-,draw,black] (r3.south) to (r7.north);
\path[-,draw,black] (r4.south) to (r2.north);
\path[-,draw,black] (r4.south) to (r7.north);
\path[-,draw,black] (r1.west) to (r7.east);
\path[-,draw,black] (r7.west) to (r2.east);

\node[box,rounded corners,very thick, minimum width=10mm,minimum height=7mm,text depth=0pt] at ($(table)+(-2.0,7.0)$) (iis) {\gls{iis}};

\node[box,rounded corners,very thick, minimum width=37mm,minimum height=7mm,text depth=0pt] at ($(iis)+(8.3,0)$) (dual) {\gls{iis} Dual Graph};

\node[box,rounded corners,very thick, minimum width=41mm,minimum height=7mm,text depth=0pt] at ($(dual)+(10.0,0)$) (greas) {Graph of Reasons};

\path[-{Stealth},draw,black] (iis.east) -- (dual.west) node[midway, above] {Relating Constraints} node[midway, below] {\textit{wrt} Scopes};
\path[-{Stealth},draw,black] (dual.east) -- (greas.west) node[midway, above] {Labelling Constraints with} node[midway, below] {\acrshort{nl} Templates (Table \ref{tab:templates_rcpsp})};

\path[-{Stealth},draw,black] (iis.south) -- (aroundtable.north);
\path[-{Stealth},draw,black] (dual.south) -- (iisgraph.north);
\path[-{Stealth},draw,black] (greas.south) -- (reasongraph.north);

\end{tikzpicture}
\end{adjustbox}
\caption{Process to build a graph-based explanation applied to our running example: (1) relating constraints within the IIS to yield its dual graph, and (2) labelling constraints in the dual graph to build the graph of reasons.}
\label{fig:megafig}
\end{figure}

\begin{table}[b]
    \caption{Templates for each type of \gls{rcpsp} constraints.}
    \label{tab:templates_rcpsp}
    \setlength{\tabcolsep}{3pt}
    \begin{tabularx}{0.9\columnwidth}{ll}
        \toprule
        (i) \textcolor{olive}{Completion}: & ``\emph{Activity j must be completed}'' \\
        (ii) \textcolor{blue}{Precedence}: & ``\emph{Act. h must be completed before Act. j starts}'' \\
        (iii) \textcolor{red}{Resource}: & ``\emph{Res. r is scarce at time $t$ due to Acts. $i,\dots,j$}'' \\
        \bottomrule
    \end{tabularx}
\end{table}

\subsection{Constructing the \emph{Graph of Reasons}}\label{subsec:ordering}

When considering a complex, large-scale optimisation problem (such as our benchmark problems in Section \ref{subsec:mip_prob}), its \gls{iis} can involve many constraints. 
Henceforth, the challenge lies in organising such constraints to assemble an understandable explanation for a user. 
With that aim, in what follows, we show that, by exploiting the structural relations among constraints in the form of a graph, we can build a ``graph-based'', structured explanation for a user whose structure can be better understood (in contrast with the unstructured list of constraints in the \gls{iis}).

To this end, we rely on well-known concepts from the Constraint Reasoning literature as follows. 
Hypergraphs \cite{dechter2022reasoning} represent how a set of constraints relates to decision variables. 
A \emph{hypergraph} 
$\mathcal{H}=(V,S)$ is a pair where $V$ is a set of vertices, and $S=\{S_1,\dots,S_l\}$, $S_i\subseteq V$, is a set of subsets of $V$ called hyperedges. Notice that the vertices represent the set of decision variables of our model, and a hyperedge $S_i$ is the scope of constraint $c_i$, which is the subset of decision variables involved in the constraint.   
However, hypergraphs do not directly represent the relations among constraints, which is, instead, the purpose of the concept of dual graphs.
Formally, a \emph{dual graph} \cite{dechter2022reasoning} $D=\{S,I\}$ is an undirected graph whose vertices $S$ are the hyperedges of $\mathcal{H}$ and $(S_i,S_j)\in I$ iff $S_i \cap S_j \neq \emptyset$. 
Intuitively, each node of a dual graph corresponds to a constraint, and each edge indicates that two constraints share some variables. 
Therefore, the dual graph of an \gls{iis} captures the relations between its constraints. Hence, we will use it as the basis for building an explanation.

Figure \ref{fig:megafig} showcases our approach for building explanations from an \gls{iis} in the context of our running example. First, we relate the constraints in an \gls{iis} to yield its dual graph. 
Despite capturing the structural relations among constraints, the dual graph of an \gls{iis} might still be challenging to understand for a user since it does not contain any \gls{nl} interpretation of the mathematical constraints.
To overcome this limitation, following the approach by \cite{pozanco2022explaining}, we label each constraint in the dual graph with an \gls{nl}
template.\footnote{Notice that translating queries and labelling constraints using \gls{nl} templates are the only problem-dependent steps that must be carried out by an expert, in contrast with the rest of our approach, which is domain-independent.
As an example, Table \ref{tab:templates_rcpsp} shows the \gls{nl} templates involved in our running example for each type of \gls{rcpsp} constraints in Section \ref{subsec:mip_prob}. } 
Figure \ref{fig:megafig} shows the constraint labelling for our running example. 
Therefore, we can finally see that a ``\emph{reason}'' takes the form of the \gls{nl} interpretation of a constraint. 
Therefore, we refer to our graph-based explanation as a \emph{graph of reasons}. 
A graph of reasons captures the structural relations among constraints inherent in the dual graph while providing human-readable semantics for each constraint.

A desirable property that characterises our explanations is that they are not ``disconnected'', i.e., they involve reasons that are always related to each other. This property can be directly expressed in terms of the graph-theory concept of \emph{connectedness}, as formally discussed in Theorem \ref{theo:con}.

\begin{theorem}\label{theo:con} 
    Given an \gls{iis}, the \gls{iis} dual graph $D$ is connected.
\end{theorem}

\begin{proof}
    We prove by contradiction that a disconnected \gls{iis} dual graph, i.e., a graph with multiple disconnected components, is not possible to obtain due to the infeasible and irreducible properties of an \gls{iis} (Definition \ref{def:iis}). 
    
    Consider an \gls{iis} graph that is disconnected in two subgraphs, $A = (C_A,E_A)$ and $B = (C_B,E_B)$ such that
    \begin{align*}
        C_A \cup C_B &= IIS, \\
        C_A \cap C_B &= \emptyset, \\
        S_A \cap S_B &= \emptyset,
    \end{align*}
    where $S_A,S_B$ are the scopes of the set of constraints $C_A$ and $C_B$ respectively. Notice that, since the graph is disconnected, $A$ and $B$ subgraphs do not share variables, i.e., $S_A \cap S_B = \emptyset$. 
    Then, since the subset of constraints $C_A$ and $C_B$ are (disjoint) subsets of the \gls{iis}, they are feasible (see Definition \ref{def:iis}). 
    However, since the both subsets of constraints have different variables (different domains), the union of both constraint sets would not affect the feasibility status of the resulting system of constraints. 
    As a result, we would have a feasible \gls{iis}, 
    which is a contradiction in itself.
    Therefore, it is not possible to have a disconnected \gls{iis} graph. 
\end{proof}

\pgfplotstableread[header=false]{
    0 7266 68.12 10.28928
    1 1597 14.97 33.43965
    2 852 7.99 43.89732
    3 421 3.95 40.25644
    4 221 2.07 33.76003
    5 111 1.04 35.94548
    6 75 0.70 28.62167
    7 48 0.45 26.12029
    8 34 0.32 19.97056
    9 18 0.17 34.04469
}\total
\pgfplotstableread[header=false]{
    0 7266 68.12 10289.28
    1 1597 14.97 33439.65
    2 852 7.99 43897.32
    3 421 3.95 40256.44
    4 221 2.07 33760.03
    5 111 1.04 35945.48
    6 75 0.70 28621.67
    7 48 0.45 26120.29
    8 34 0.32 19970.56
    9 18 0.17 34044.69
}\totalnodiv

\pgfplotsset{
    select row/.style={
        x filter/.code={\ifnum\coordindex=#1\else\def\pgfmathresult{}\fi}
    }
}

\definecolor{myred}{HTML}{af0926}
\definecolor{myblue}{HTML}{80e2ff}
\definecolor{mygreen}{HTML}{1e9a51}

\def\figh{45mm}

\begin{figure}[b]
    \centering
    \begin{tikzpicture}
        \begin{axis} [
            width=0.9\columnwidth,
            height=\figh,
            ybar,
            bar width = 9pt,
            ymajorgrids = false,
            ylabel = {
            (\%)
            },
            ymin = 0, 
            ymax = 70,
            ytick distance = 20,
            y label style={at={(-0.07,0.5)},},
            enlarge x limits=0.07,
            xlabel = {\small IIS size reduction},
            xtick={0,...,9},
            xticklabels from table={\total}{0},
            axis y line*=left,
            legend style={
                at={(0.01,0.98)},
                anchor=north west,
            },
            legend cell align={left},
            y tick label style={
                /pgf/number format/fixed,
                /pgf/number format/fixed zerofill,
                /pgf/number format/precision=0
            },
        ]
        
        \pgfplotsinvokeforeach{0,...,9}{
            \addplot[fill=mygreen!50,postaction={pattern=north east lines},bar shift=-4.5pt] table [select row=#1, x expr=#1, y=2] {\total};
        }
        \end{axis}
        \begin{axis} [
            width=0.9\columnwidth,
            height=\figh,
            ybar,
            bar width = 9pt,
            ymajorgrids = true,
            grid style=dashed,
            ylabel = {
            ($\times 10^{3}$)
            },
            ymin = 0,
            ymax = 70,
            ytick={0,20,...,60},
            y label style={at={(1.07,0.5)},rotate=180,},
            enlarge x limits=0.07,
            xtick={11,12,...,20},
            xticklabels from table={\total}{0},
            axis y line*=right,
            legend style={
                at={(1.0,1.0)},
                anchor=north east,
            },
            legend cell align={left},
        ]
        \pgfplotsinvokeforeach{0,...,9}{
            \addplot[fill=myred!50,postaction={pattern=crosshatch dots},bar shift=4.5pt] table [select row=#1, x expr=#1, y=3] {\total};
        }
        \end{axis}
    \end{tikzpicture}
    \caption{Distribution of \gls{iis} size reduction between those computed by FORQES and those computed by CPLEX (\textcolor{mygreen}{lines}, left y-axis) and runtime increase to compute the smallest IIS (\textcolor{myred}{dots}, right y-axis).}
    \label{fig:smallest_iis_comb}
\end{figure}
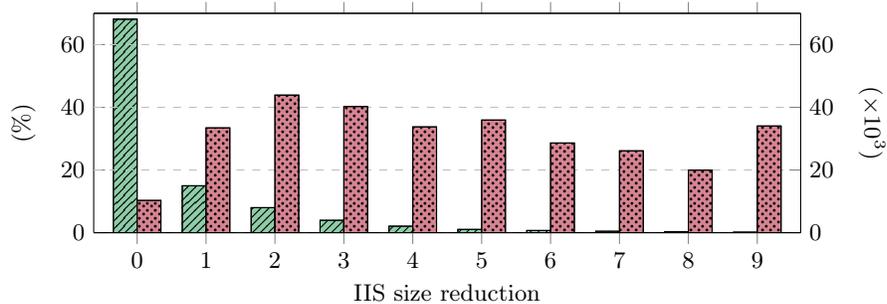

\section{Experimental Evaluation}
\label{sec:exp}

The main goal of this section is to evaluate the empirical hardness \cite{leyton2002learning} of computing explanations. 
As a benefit, this allows us to characterise problem instances for which we can compute explanations in real time for users. 
Our study will measure (i) the times needed to compute \gls{iis}s to build explanations, and (ii) the \emph{time overheads} of building explanations with respect to solving optimisation problems. Thus, given a problem and a query, we define the \emph{time overhead} to build an explanation for the query as the ratio between the time to compute the \gls{iis} and the time to compute the optimal solution for the problem. Therefore, our study compares the hardness of solving optimisation problems with respect to the hardness of computing their explanations.

An additional goal is to empirically learn whether the computational effort required to compute the smallest \gls{iis} with state-of-the-art techniques pays off. This would be the case if we do obtain \glspl{iis} that are significantly smaller than those computed by CPLEX, our off-the-shelf solver of choice. With this aim, we compare the \gls{iis} sizes (number of constraints) obtained and the runtimes required by both methods.

\subsection{Experimental Setting}\label{subsec:data}

We solve all problem instances and compute \glspl{iis} with CPLEX v22.1.0.\footnote{Code for the experimental evaluation at: \url{https://github.com/RogerXLera/ExMIPExperimentalEvaluation}.} 
We employ FORQES \cite{ignatiev2015smallest} to compute the smallest \gls{iis} because, to the best of our knowledge, is the most recent algorithm for computing the smallest \gls{iis}.\footnote{FORQES was originally designed to compute the smallest \gls{mus}, a concept closely related to \glspl{iis} (see Section \ref{sec:related}). To compute the smallest \gls{mus}, FORQES iteratively computes the \gls{mcs} of different sets of constraints. We adapt FORQES by iteratively computing the \gls{mfs} instead, an analogous problem to \gls{mcs} in the constraint programming literature \cite{chinneck2007feasibility}. Our adaptation does not affect the complexity of the algorithm.} 
We run our tests on a machine with a 10-core 2.2GHz CPU and 16GB of RAM. We limit solving time to 10 hours per problem instance and per explanation.

\subsubsection{Problem Instances}\label{subsec:prob_instances}
\paragraph{\gls{rcpsp}}
Among the many variants of the \gls{rcpsp}, 
here we consider the single-mode \gls{rcpsp}, where each activity can only be completed in one way. 
We use 1440 problem instances with different number of activities $|J|\in \{30,60,90\}$ from the \gls{psplib} \cite{kolisch1997psplib}.\footnote{\url{https://www.om-db.wi.tum.de/psplib/} (April 2024).}

\paragraph{\gls{wdp}}
We use the \gls{cats} \cite{leyton2000towards} to generate realistic instances of \gls{ca}.\footnote{\url{https://github.com/kevinlb1/CATS}.}
We consider 4 realistic distributions (\emph{paths, regions, matching,} and \emph{scheduling}) to generate bids and we vary the number of bids $|B| \in \{20,50,100,200,500,1000\}$ for 10 different seeds, i.e., we solve a total of 240 instances.\footnote{We vary the number of goods according to the number of bids such that $|G| \in \{12,20,50,100,200,500\}$.}

Once all problem instances are solved, we generate queries as described in Section \ref{sec:query} (8 query types both for the \gls{rcpsp} and the \gls{wdp}). Therefore, we compute an \gls{iis} for each combination of query type and problem instance, both for the \gls{rcpsp} and the \gls{wdp} (11,520 and 1,920 \glspl{iis} respectively).

\definecolor{myred}{HTML}{af0926}
\definecolor{myblue}{HTML}{80e2ff}
\definecolor{mygreen}{HTML}{1e9a51}

\def\figh{52mm}
\def\figw{0.7\columnwidth}

\begin{figure}[t]
    \centering
    \begin{tikzpicture}
        \begin{semilogyaxis} [
            boxplot/draw direction=y,
            width=\figw,
            height=\figh,
            xmin = 0.5,
            xmax = 3.5,
            ymajorgrids = false,
            ylabel = Time (s),
            ymin = 0.005, 
            ymax = 40000,
            ytick = {0.01,0.1,1,10,100,1000,10000},
            yticklabels = {$10^{-2}$,$10^{-1}$,$10^{0}$,$10^{1}$,$10^{2}$,$10^{3}$,$10^{4}$},
            y label style={at={(-0.12,0.5)},},
            enlarge x limits=0.07,
            xtick={1,2,3},
            xticklabels = {, , },
            axis y line*=left,
            legend style={
                at={(0.01,0.98)},
                anchor=north west,
            },
            legend cell align={left},
            y tick label style={
                /pgf/number format/fixed,
                /pgf/number format/fixed zerofill,
                /pgf/number format/precision=0
            },
        ]
        \addplot[mark=*,fill=myblue,boxplot prepared={
                lower whisker=0.380,
                lower quartile=12.120,
                median=31.170,
                upper quartile=918.060,
                upper whisker=36000.0,
            },] coordinates {};
        \addplot[mark=*,fill=mygreen,boxplot prepared={
                lower whisker=0.019,
                lower quartile=0.745,
                median=3.252,
                upper quartile=53.398,
                upper whisker=19000.0,
            },] coordinates {};
        \addplot[mark=*,fill=myred,boxplot prepared={
                lower whisker=0.0019,
                lower quartile=0.0019,
                median=0.0019,
                upper quartile=0.0019,
                upper whisker=0.0019,
            },] coordinates {};
        
        \end{semilogyaxis}
        \begin{axis} [
            boxplot/draw direction=y,
            width=\figw,
            height=\figh,
            xmin = 0.5,
            xmax = 3.5,
            ymajorgrids = false,
            grid style=dashed,
            ylabel = {Overhead (\%)},
            ymin = -5, 
            ymax = 45,
            ytick distance = 10,
            y label style={at={(1.10,0.5)},rotate=180,},
            enlarge x limits=0.07,
            xtick={1,2,3},
            xticklabels = {\gls{milp}, \gls{iis},Overhead},
            axis y line*=right,
            legend style={
                at={(0.02,0.98)},
                anchor=north west,
            },
            legend cell align={left},
            y tick label style={
                /pgf/number format/fixed,
                /pgf/number format/fixed zerofill,
                /pgf/number format/precision=0
            },
        ]
        \addplot[mark=*,fill=mygreen,boxplot prepared={
                lower whisker=-5.019,
                lower quartile=-5.745,
                median=-5.252,
                upper quartile=-53.398,
                upper whisker=-132.376,
            },] coordinates {};
        \addplot[mark=*,fill=myblue,boxplot prepared={
                lower whisker=-5.019,
                lower quartile=-5.745,
                median=-5.252,
                upper quartile=-53.398,
                upper whisker=-132.376,
            },] coordinates {};
        \addplot[mark=*,fill=myred,boxplot prepared={
                average=327.223,
                lower whisker=0.000,
                lower quartile=1.271,
                median=6.732,
                upper quartile=17.064,
                upper whisker=40.754,
            },] coordinates {};
        \end{axis}
    \end{tikzpicture}
    \caption{Runtimes to solve \gls{milp} problems and compute \glspl{iis} (left y-axis), and time overhead (right y-axis), for \gls{rcpsp}.}
    \label{fig:boxplot_rcpsp}
\end{figure}
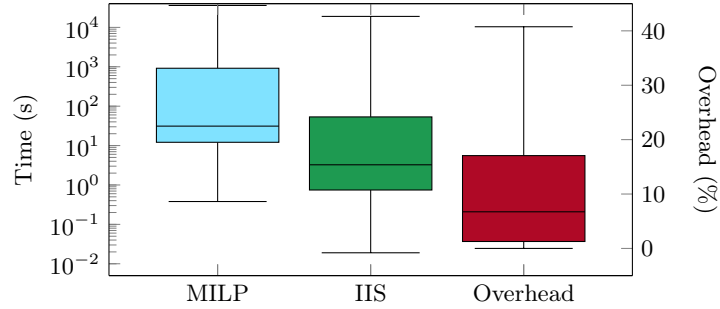

\pgfplotstableread[header=false]{
    Q2 9.384 0.328
    Q6 10.321 0.732
    Q4 20.075 1.263
    Q5 34.579 1.985
    Q7 36.010 2.243
    Q1 36.827 2.136
    Q8 39.204 2.617
    Q3 68.036 4.284
}\iistime
\pgfplotstableread[header=false]{
    Q2 5.489 0.758
    Q6 5.219 0.171
    Q4 7.607 0.234
    Q5 8.353 0.211
    Q7 8.604 0.260
    Q1 8.492 0.233
    Q8 10.930 0.499
    Q3 15.792 0.375
}\overhead

\pgfplotsset{
    select row/.style={
        x filter/.code={\ifnum\coordindex=#1\else\def\pgfmathresult{}\fi}
    }
}

\definecolor{myred}{HTML}{af0926}
\definecolor{myblue}{HTML}{80e2ff}
\definecolor{mygreen}{HTML}{1e9a51}

\def\figh{52mm}
\def\figw{0.9\columnwidth}

\begin{figure}[t]
    \centering
    \begin{tikzpicture}
        \begin{axis} [
            width=\figw,
            height=\figh,
            ybar,
            bar width = 9pt,
            ymajorgrids = false,
            ylabel = {$IIS$ runtime (s)},
            ymin = 0, 
            ymax = 80,
            ytick distance = 20,
            y label style={at={(-0.1,0.5)},},
            enlarge x limits=0.07,
            xtick={0,...,7},
            xticklabels from table={\iistime}{0},
            axis y line*=left,
            legend style={
                at={(0.01,0.98)},
                anchor=north west,
            },
            legend cell align={left},
            y tick label style={
                /pgf/number format/fixed,
                /pgf/number format/fixed zerofill,
                /pgf/number format/precision=0
            },
        ]
        
        \pgfplotsinvokeforeach{0,...,7}{
            \addplot[fill=mygreen!50,postaction={pattern=north east lines},bar shift=-4.5pt,error bars/.cd, y dir=both, y explicit] table [select row=#1, x expr=#1, y=1,y error=2] {\iistime};
        }
        
        \end{axis}
        \begin{axis} [
            width=\figw,
            height=\figh,
            ybar,
            bar width = 9pt,
            ymajorgrids = true,
            grid style=dashed,
            ylabel = {Overhead (\%)},
            ymin = 0, 
            ymax = 20,
            ytick distance = 5,
            y label style={at={(1.1,0.5)},rotate=180,},
            enlarge x limits=0.07,
            xtick={20,...,28},
            xticklabels from table={\overhead}{0},
            axis y line*=right,
            legend style={
                at={(0.02,0.98)},
                anchor=north west,
            },
            legend cell align={left},
        ]
        \pgfplotsinvokeforeach{0,...,7}{
            \addplot[fill=myred!50,postaction={pattern=crosshatch dots},bar shift=4.5pt,error bars/.cd, y dir=both, y explicit] table [select row=#1, x expr=#1, y=1,y error=2] {\overhead};
        }
        \end{axis}
    \end{tikzpicture}
    \caption{Average \gls{iis} computation time (\textcolor{mygreen}{lines}, left y-axis) and overhead (\textcolor{myred}{dots}, right y-axis) for different query types for the \gls{rcpsp}.}
    \label{fig:barplot_rcpsp}
\end{figure}
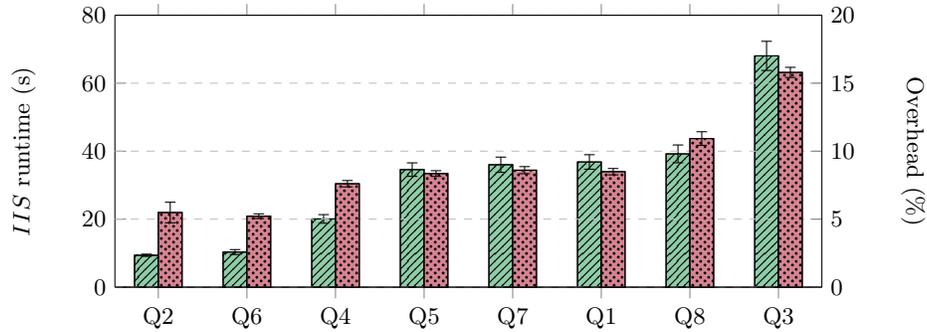
\normalsize

\subsection{Results}\label{subsec:results}

\subsubsection{Smallest \gls{iis}}\label{subsec:results:smalliis}

We compute the \glspl{iis} resulting from all generated queries with both algorithms (i.e., CPLEX and FORQES), and we compare both algorithms' runtimes as well as the difference in the sizes of the computed \glspl{iis}. 
Figure \ref{fig:smallest_iis_comb} shows such a difference in \glspl{iis}' sizes and the mean runtime increase of computing the smallest \gls{iis} with FORQES. We report the runtime increase as the ratio between FORQES's runtime and CPLEX's runtime. 
Results show that for the majority of instances ($\sim$$70\%$), both algorithms obtain \glspl{iis} of the same size. FORQES computes \gls{iis}s with a size reduction greater than 2 constraints in only $7\%$ of the cases. 
In contrast, FORQES requires $4\cdot 10^{4}$ times more runtime than CPLEX on average. 
Overall, our results show that computing the smallest \gls{iis} via FORQES does \emph{not} result in \glspl{iis} that are significantly smaller than those computed by CPLEX, which, in contrast, terminates in a significantly shorter runtime.
Henceforth, in the rest of our experimental evaluation, we only employ CPLEX to compute \glspl{iis}.

\subsubsection{Time overhead when computing explanations}

\paragraph{\gls{rcpsp}} Figure \ref{fig:boxplot_rcpsp} shows the \gls{milp} solving times, the \gls{iis} computation times, and the time overhead to compute explanations for \gls{rcpsp} problem instances. Notice that computing \glspl{iis} is significantly faster than solving \gls{milp} problems. We observe that the overhead is, for most cases, lower than $40\%$, indicating that most explanations can be computed faster than half the time required to solve a problem instance. 
On the one hand, the median overhead to compute an explanation is $6.7\%$, and the median time is $3.3$ seconds. On the other hand, the median \gls{milp} solving time is $31.2$ seconds. Finally, we compute more than $75\%$ of the explanations in less than 1 minute. Notice that generating fast explanations is a key requirement for the users of explainable systems to allow them to interact according to the study conducted by \cite{lakkaraju2022rethinking}. 
Figure \ref{fig:barplot_rcpsp} shows the impact of the query type (see Table \ref{tab:queries_rcpsp}) on the average runtime to compute an \gls{iis}, and consequently, the required overhead to compute an explanation. On the one hand, most query types are solved on average in less than 60 seconds and have less than $10\%$ of time overhead. 
On the other hand, some query types are significantly harder than others, e.g., Q3. It remains future work to establish why such a query is harder to explain than others.

\pgfplotstableread[header=false]{
sch.\vphantom{pgth} 0.7 0.1
reg.\vphantom{pgth} 0.8 0.1
mat.\vphantom{pgth} 0.9 0.1
pat.\vphantom{pgth} 1.4 0.2
}\iistimesmall
\pgfplotstableread[header=false]{
scheduling 0.554 0.073
regions 1.193 0.146
matching 1.341 0.097
paths 1.661 0.208
}\overheadsmall
\pgfplotstableread[header=false]{
sch.\vphantom{pgth} 0.007 0.000
mat.\vphantom{pgth} 0.621 0.300
pat.\vphantom{pgth} 9.976 15.125
reg.\vphantom{pgth} 553.564 897.712
}\iistimemedium
\pgfplotstableread[header=false]{
scheduling 0.00806 0.00232
matching 0.55741 0.22286
paths 9.46089 4.65204
regions 19.19069 5.13679
}\overheadmedium
\pgfplotstableread[header=false]{
sch.\vphantom{pgth} 0.022 0.004
mat.\vphantom{pgth} 298.535 45.072
pat.\vphantom{pgth} 5944.190 1480.373
reg.\vphantom{pgth} 36000.181 1656.709
}\iistimelarge
\pgfplotstableread[header=false]{
scheduling 0.04304 0.01160
matching 20.010669 3.353928
paths 39.40299 6.11688
regions 39.75008 2.64844
}\overheadlarge

\pgfplotsset{
    select row/.style={
        x filter/.code={\ifnum\coordindex=#1\else\def\pgfmathresult{}\fi}
    }
}

\definecolor{myred}{HTML}{af0926}
\definecolor{myblue}{HTML}{80e2ff}
\definecolor{mygreen}{HTML}{1e9a51}

\def\figh{45mm}
\def\figw{0.49\textwidth}
\small
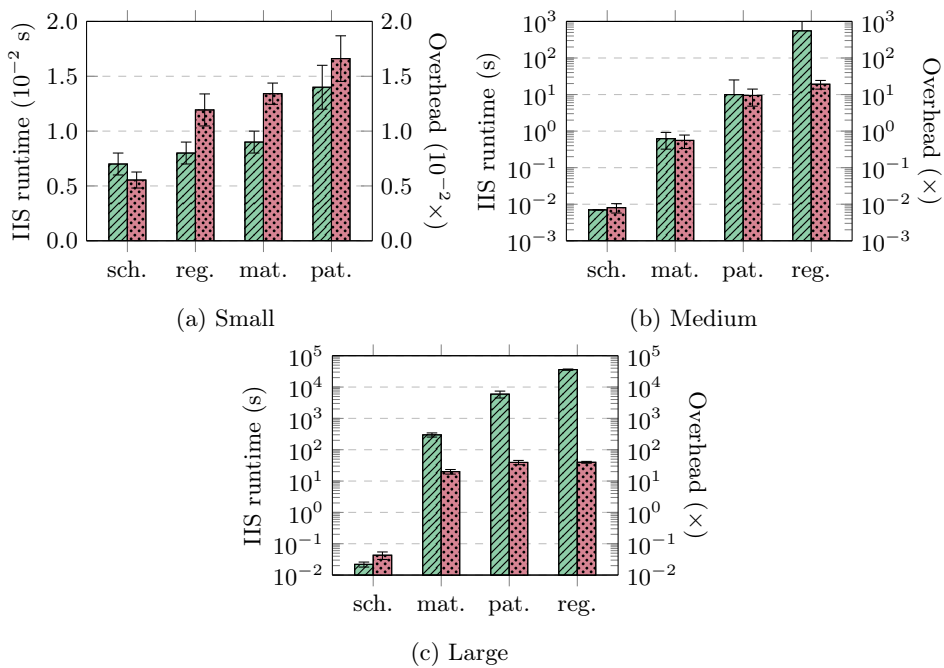
\begin{figure*}[t]
    \captionsetup[subfigure]{justification=centering}
    \centering
    \begin{subfigure}{\figw}
        \centering
    \begin{tikzpicture}
        \begin{axis} [
            width=0.90\textwidth,
            height=\figh,
            ybar,
            bar width = 7pt,
            ymajorgrids = true,
            grid style=dashed,
            ylabel = {\footnotesize IIS runtime ($10^{-2}$ s)},
            ymin = 0, 
            ymax = 2,
            ytick distance = 0.5,
            y label style={at={(-0.15,0.5)},},
            enlarge x limits=0.2,
            xtick={0,...,3},
            xticklabels from table={\iistimesmall}{0},
            axis y line*=left,
            legend style={
                at={(0.01,0.98)},
                anchor=north west,
            },
            legend cell align={left},
            y tick label style={
                /pgf/number format/fixed,
                /pgf/number format/fixed zerofill,
                /pgf/number format/precision=1
            },
        ]
        \pgfplotsinvokeforeach{0,...,3}{
            \addplot[fill=mygreen!50, postaction={pattern=north east lines},bar shift=-3.5pt,error bars/.cd, y dir=both, y explicit] table [select row=#1, x expr=#1, y=1,y error=2] {\iistimesmall};
        }
        \end{axis}
        \begin{axis} [
            width=0.90\textwidth,
            height=\figh,
            ybar,
            bar width = 7pt,
            ylabel = {\footnotesize Overhead ($10^{-2}\times$)},
            ymin = 0, 
            ymax = 2.0,
            ytick distance = 0.5,
            y label style={at={(1.15,0.5)},rotate=180,},
            enlarge x limits=0.2,
            xtick={8,...,11},
            xticklabels from table={\overheadsmall}{0},
            axis y line*=right,
            legend style={
                at={(0.0,1.0)},
                anchor=north west,
            },
            legend cell align={left},
            y tick label style={
                /pgf/number format/fixed,
                /pgf/number format/fixed zerofill,
                /pgf/number format/precision=1
            },
        ]
        \pgfplotsinvokeforeach{0,...,3}{
            \addplot[fill=myred!50,postaction={pattern=crosshatch dots},bar shift=3.5pt,error bars/.cd, y dir=both, y explicit] table [select row=#1, x expr=#1, y=1,y error=2] {\overheadsmall};
        }
        \end{axis}
    \end{tikzpicture}
    \subcaption{Small}
    \label{fig:dist_small}
    \end{subfigure}
    \hfill
    \begin{subfigure}{\figw}
    \centering
    \begin{tikzpicture}
        \begin{semilogyaxis} [
            width=0.90\textwidth,
            height=\figh,
            ybar,
            log origin=infty,
            bar width = 7pt,
            ymajorgrids = true,
            grid style=dashed,
            ylabel = {\footnotesize IIS runtime (s)},
            ymin = 0.001, 
            ymax = 1000,
            ytick = {0.001,0.01,0.1,1,10,100,1000},
            log plot exponent style/.style={
                /pgf/number format/precision=0,
            }, 
            y label style={at={(-0.20,0.5)},},
            enlarge x limits=0.2,
            xtick={0,...,3},
            xticklabels from table={\iistimemedium}{0},
            axis y line*=left,
            legend style={
                at={(0.01,0.98)},
                anchor=north west,
            },
            legend cell align={left},
            y tick label style={
                /pgf/number format/fixed,
                /pgf/number format/fixed zerofill,
                /pgf/number format/precision=0
            },
        ]
        \pgfplotsinvokeforeach{0,...,3}{
            \addplot[fill=mygreen!50,postaction={pattern=north east lines},bar shift=-3.5pt,error bars/.cd,y dir=both, y explicit] table [select row=#1, x expr=#1, y=1,y error=2] {\iistimemedium};
        }
        \end{semilogyaxis}
        \begin{semilogyaxis} [
            width=0.90\textwidth,
            height=\figh,
            ybar,
            log origin=infty,
            bar width = 7pt,
            ylabel = {\footnotesize Overhead ($\times$)},
            ymin = 0.001, 
            ymax = 1000,
            ytick = {0.001,0.01,0.1,1,10,100,1000},
            log plot exponent style/.style={
                /pgf/number format/precision=0,
            }, 
            y label style={at={(1.2,0.5)},rotate=180,},
            enlarge x limits=0.2,
            xtick={8,...,11},
            xticklabels from table={\overheadmedium}{0},
            axis y line*=right,
            legend style={
                at={(0.0,1.0)},
                anchor=north west,
            },
            legend cell align={left},
            y tick label style={
                /pgf/number format/fixed,
                /pgf/number format/fixed zerofill,
                /pgf/number format/precision=1
            },
        ]
        \pgfplotsinvokeforeach{0,...,3}{
            \addplot[fill=myred!50,postaction={pattern=crosshatch dots},bar shift=3.5pt,error bars/.cd, y dir=both, y explicit] table [select row=#1, x expr=#1, y=1,y error=2] {\overheadmedium};
        }
        \end{semilogyaxis}
    \end{tikzpicture}
    \subcaption{Medium}
    \label{fig:dist_medium}
    \end{subfigure}
    \hfill
    \begin{subfigure}{\figw}
    \centering
    \begin{tikzpicture}
        \begin{semilogyaxis} [
            width=0.90\textwidth,
            height=\figh,
            ybar,
            log origin=infty,
            bar width = 7pt,
            ymajorgrids = true,
            grid style=dashed,
            ylabel = {\footnotesize IIS runtime (s)},
            ymin = 0.01, 
            ymax = 100000,
            ytick = {0.01,0.1,1,10,100,1000,10000,100000},
            log plot exponent style/.style={
                /pgf/number format/precision=0,
            }, 
            y label style={at={(-0.20,0.50)},},
            enlarge x limits=0.2,
            xtick={0,...,3},
            xticklabels from table={\iistimelarge}{0},
            axis y line*=left,
            legend style={
                at={(0.01,0.98)},
                anchor=north west,
            },
            legend cell align={left},
            y tick label style={
                /pgf/number format/fixed,
                /pgf/number format/fixed zerofill,
                /pgf/number format/precision=0
            },
        ]
        \pgfplotsinvokeforeach{0,...,3}{
            \addplot[fill=mygreen!50,postaction={pattern=north east lines},bar shift=-3.5pt,error bars/.cd, y dir=both, y explicit] table [select row=#1, x expr=#1, y=1,y error=2] {\iistimelarge};
        }
        \end{semilogyaxis}
        \begin{semilogyaxis} [
            width=0.90\textwidth,
            height=\figh,
            ybar,
            log origin=infty,
            bar width = 7pt,
            ylabel = {\footnotesize Overhead ($\times$)},
            ymin = 0.01, 
            ymax = 100000,
            ytick = {0.01,0.1,1,10,100,1000,10000,100000},
            log plot exponent style/.style={
                /pgf/number format/precision=0,
            }, 
            y label style={at={(1.20,0.50)},rotate=180,},
            enlarge x limits=0.2,
            xtick={8,...,11},
            xticklabels from table={\overheadlarge}{0},
            axis y line*=right,
            legend style={
                at={(0.0,1.07)},
                anchor=north west,
            },
            legend cell align={left},
            y tick label style={
                /pgf/number format/fixed,
                /pgf/number format/fixed zerofill,
                /pgf/number format/precision=0
            },
        ]
        \pgfplotsinvokeforeach{0,...,3}{
            \addplot[fill=myred!50,postaction={pattern=crosshatch dots},bar shift=3.5pt,error bars/.cd, y dir=both, y explicit] table [select row=#1, x expr=#1, y=1,y error=2] {\overheadlarge};
        }
        \end{semilogyaxis}
    \end{tikzpicture}
    \subcaption{Large}
    \label{fig:dist_large}
    \end{subfigure}
    \caption{Average \gls{iis} runtimes (\textcolor{mygreen}{lines}) and overhead (\textcolor{myred}{dots}) for different \gls{wdp} instance distributions: \emph{scheduling}, \emph{matching}, \emph{paths}, and \emph{regions}.}
    \label{fig:distributions}
\end{figure*}

\normalsize

\paragraph{\gls{wdp}} Figure \ref{fig:distributions} shows the average \gls{iis} computation time and overhead for each distribution. We plot the overhead with ``$\times$'' units, indicating the cost of computing \glspl{iis} compared to the time to solve problem instances. 
We organise the results in three plots of different instance sizes based on the number of bids $|B|$: small ($|B| \in \{20,50\}$), medium ($|B| \in \{100,200\}$), and large ($|B| \in \{500,1000\}$). We generally observe that as instance sizes increase, average \gls{iis} computation times and overheads increase too. 
In detail, for small-size instances (Fig. \ref{fig:dist_small}), computing explanations is much cheaper than solving problems (i.e., the overhead is $\ll 1 \times$), and fast (no more than $2\cdot 10^{-2}$ seconds on average).
Moreover, there are significant differences in the average \gls{iis} computation time for the different distribution types for medium-size instances (Fig. \ref{fig:dist_medium}). While computing an \gls{iis} is fast for the scheduling and matching distributions (less than one second), it is much harder for the region's distribution (550 seconds and an overhead of $20\times$ on average).
Finally, computing an \gls{iis} is hard for large instances, and hence the overhead increases for almost all distributions except the scheduling distribution. 
Notice that we observe differences between distributions in \gls{iis} computation time. Such differences have also been reported in the empirical hardness of solving instances \cite{leyton2002learning}. 

\section{Conclusions and future work}\label{sec:conclusion}

In this paper, we aim to build fundamental algorithmic tools for generating explanations for optimisation. We introduced \acrlong{exmip}, an approach to building contrastive explanations to answer queries posed by a user. Our goal is to provide them explanations presenting the reasons for which their query leads to a worse solution. 
Such reasons are then presented in natural language as a \emph{graph of reasons}, whose structure helps a user to understand the relations between reasons. 
We evaluate our approach with two well-known optimisation problems: the \gls{rcpsp} and the \gls{wdp}. Our results indicate that for the \gls{rcpsp}, our method computes explanations much faster than solving an \gls{milp} problem. This is a key aspect of \gls{exmip} since state-of-the-art approaches require solving additional \gls{milp} problems to generate explanations \cite{georgara2022building,zehtabi2024contrastive}. For the \gls{wdp}, we can compute explanations in less than one minute for small and medium instances, though computing explanations for large instances can be more challenging.

As future work, we plan to predict the expected runtime required to compute an explanation for a given problem instance. 
In addition, we plan to apply \gls{exmip} to other optimisation-based application domains such as \textit{Bachelor's Degree Planning} \cite{lera2025computing}, \acrlong{rpp} \cite{lera2024robust} or \textit{Value Systems Aggregation} \cite{lera2024aggregating}.

\begin{credits}
\subsubsection{\ackname} 
The authors were supported by the research project ACISUD (PID2022-136787NB-I00).

\subsubsection{\discintname}
The authors have no competing interests to declare that are
relevant to the content of this article.
\end{credits}
%
%
%
\bibliographystyle{splncs04}
\bibliography{sample}

\end{document}